\documentclass[11pt,letterpaper]{article}
\usepackage{amsmath,amssymb,fullpage,paralist,algorithm2e}

\newtheorem{theorem}{Theorem}

\newtheorem{definition}[theorem]{Definition}

\def\FullBox{\hbox{\vrule width 8pt height 8pt depth 0pt}}

\newcommand{\qed}{\;\;\;\FullBox}

\newenvironment{proof}{\noindent{\bf Proof:~~}}{\qed}

\newcommand{\bbE}{\mathbb{E}}
\newcommand{\bbR}{\mathbb{R}}

\newcommand{\calC}{\mathcal{C}}

\newcommand{\poly}{\mathrm{poly}}
\newcommand{\NP}{\mathsf{NP}}

\title{Conditional Sparse Linear Regression}
\author{Brendan Juba\thanks{Supported by an AFOSR Young Investigator Award.}\\
Washington University in St. Louis\\{\tt bjuba@wustl.edu}}

\begin{document}

\maketitle

\begin{abstract}
Machine learning and statistics typically focus on building models that capture 
the vast majority of the data, possibly ignoring a small subset of data as 
``noise'' or ``outliers.'' By contrast, here we consider the problem of 
{\em jointly} identifying a significant (but perhaps small) segment of a 
population in which there is a highly sparse linear regression fit, together 
with the coefficients for the linear fit. We contend that such tasks are of 
interest both because the models themselves may be able to achieve better
predictions in such special cases, but also because they may aid our 
understanding of the data. We give algorithms for such problems under the sup 
norm, when this unknown segment of the population is described by a $k$-DNF 
condition and the regression fit is $s$-sparse for constant $k$ and $s$. For the
variants of this problem when the regression fit is {\em not} so sparse or using
expected error, we also give a preliminary algorithm and highlight the question 
as a challenge for future work.
\end{abstract}


\section{Introduction}
{\em Linear regression,} the fitting of linear relationships among variables in 
a data set, is a standard tool in data analysis. In particular, for the sake of
interpretability and utility in further analysis, we desire to find {\em highly 
sparse} linear relationships, i.e., involving only a few variables. 
Of course, such simple linear relationships often will not hold across an entire
population. But, more frequently there will exist conditions -- perhaps a range
of parameters or a segment of a larger population -- under which such sparse 
models fit the data quite well. For example, Rosenfeld et al.~\cite{rgh+15}
used data mining heuristics to identify small segments of a population in which 
a few additional risk factors were highly predictive of certain kinds of cancer,
whereas these same risk factors were not significant in the overall population. 
Simple rules for special cases may also hint at the more complex general rules.
More generally, we need to develop new techniques to reason about populations in
which most members are atypical in some way, which are colloquially (and 
somewhat abusively) referred to as {\em long-tailed} distributions. We are 
seeking principled alternatives to ad-hoc approaches such as trying a variety 
of methods for clustering the data and hoping that the identified clusters can
be modeled well.

\subsection{Our results}
In this work we consider the design and analysis of efficient 
algorithms for the {\em joint} task of identifying significant segments of a 
population in which a sparse model provides a good fit. We are able to identify
such segments when they are described by a $k$-DNF and there is a $s$-sparse
regression fit for constant $k$ and $s$. More specifically, we give algorithms
when there is a linear relationship with respect to which the error is bounded 
by $\epsilon$ with probability~$1$ (i.e., $\epsilon$ sup norm). In this case, we
find a condition in which the error is bounded by $\epsilon$ for a $1-\gamma$ 
fraction of the population (with probability $1-\delta$ over the sample of 
data).
\begin{theorem}[Conditional sparse linear regression]\label{mainthm}
Suppose that $D$ is a joint probability distribution over $\vec{x}\in\{0,1\}^n$,
$\vec{y}\in\bbR^d$, and $z\in\bbR$ such that there is a $k$-DNF $c$ for
which for some $s$-sparse $\vec{a}\in\bbR^d$ 
\[
\Pr_{(x,y,z)\in D}\left[|\langle\vec{a},\vec{y}\rangle-z|\leq\epsilon
\mathlarger{\mathlarger{|}}c(\vec{x})=1\right]=1
\qquad\text{and}\qquad
\Pr_{(x,y,z)\in D}[c(\vec{x})=1]\geq\mu.
\]
Then given $\epsilon$, $\mu$, and $\delta$ in $(0,1)$, $\gamma\in (0,1/2]$, and 
access to examples from $D$, for any constants $s$ and $k$, there is an 
algorithm that runs in polynomial time in $n$, $d$, $1/\mu$, $1/\gamma$, and
$\log 1/\delta$, and finds an $s$-sparse $\vec{a}'$ and $k$-DNF $c'$ such that 
with probability $1-\delta$,
\[
\Pr_{(x,y,z)\in D}\left[|\langle\vec{a}',\vec{y}\rangle-z|\leq\epsilon
\mathlarger{\mathlarger{|}}c'(\vec{x})=1\right]\geq 1-\gamma
\qquad\text{and}\qquad
\Pr_{(x,y,z)\in D}[c'(\vec{x})=1]\geq(1-\gamma)\mu.
\]
\end{theorem}

Our algorithms make crucial use of the sought solution's sparsity. The key
observation is that since the linear rule has constant sparsity, with respect
to the relevant dimensions there are a constant number of ``extremal examples''
such that we can obtain low error on the unknown event by fitting these extremal
examples. We can then use the linear rule we obtain from fitting such a set of 
examples to label the data according to whether or not that point has low error 
under the linear rule. Finally, this enables us to find an event on which the 
linear rule has low error. Thus, it suffices to simply perform a search over 
candidates for the extremal examples and return one for which the corresponding 
event captures enough of the data.

We also note a trivial (weak) approximation algorithm for an expected-error 
variant of the problem that does not rely on sparsity: when there is a $k$-DNF
$c$ and a linear rule $a$ giving conditional expected error $\epsilon$ (and
$c$ is true with probability $\mu$), we find a condition $c'$ and a linear rule 
$a'$ with conditional expected error $O(n^k\epsilon)$ and probability 
$\Omega(\mu/n^k)$. We pose the design of better algorithms for the dense 
regression and expected-error tasks as challenges for future work.

\subsection{Related work}
We are building on recent work by Juba~\cite{juba16} on identifying potentially
rare events of interest in a distribution, which captures a family of data
mining tasks (including, e.g., association rule discovery~\cite{amstv96} or
``bump hunting''~\cite{ff99}). This work is closely related to theoretical work 
on {\em positive-reliable learning}~\cite{kkm12,kt14}, which is in turn very 
closely related to the ``heuristic learning'' model introduced by Pitt and 
Valiant~\cite{pv88} and studied in depth by Bshouty and Burroughs~\cite{bb05}: 
these are models of classification in which one type of error is minimized 
subject to a hard bound on the other type of error. The key difference is 
essentially that the work by Juba, like the work in data mining, focuses on 
bounding the error {\em conditioned on the identified event}. In the present 
work, we develop this perspective further, and seek to perform {\em supervised
learning} in such a conditional distribution. With respect to the current work, 
these earlier works can be viewed with hindsight as identifying a conditional 
distribution in which the class consisting solely of the constant~$1$ 
function fits the identified conditional distribution with low error.
Presently, we generalize this to the problem of fitting a (sparse) linear 
rule in the identified conditional distribution.

Our work also has some relationship to the enormous body of work on {\em robust 
statistics}~\cite{huber81,rl87}, in which {\em outliers} are identified and 
ignored or otherwise mitigated. The difference in what we consider here is 
two-fold. First, we are specifically interested in the case where {\em we may 
decline to fit the vast majority of the data,} thus treating most of the data as
``outliers'' in the model of robust statistics. Second, we are also interested 
in {\em finding a (simple) rule that identifies which subset of the data we are 
fitting} (and which subset we are ignoring). By contrast, in robust statistics, 
an arbitrary subset of the data may be considered ``corrupted'' and ignored.
Our problem is also very closely related to the problem solved by algorithms 
such as RANSAC~\cite{fb81} that can find nontrivial linear relationships in data
even when these are only of moderate density. The difference is principally that
RANSAC is designed to find linear relationships in very low dimension (e.g., in 
$\bbR^2$), and does not scale to high dimensions. Here, by contrast, although
the linear fit we are seeking is of constant sparsity, we wish to find linear 
relationships in asymptotically growing dimension $d$. Also, RANSAC-like 
algorithms, like work in robust statistics, do not aim to provide a 
description of the data for which they find a linear relationship.

\section{Problem definition and background}\label{def-relationships}
In this work, we primarily focus on the following task:

\begin{definition}[Conditional linear regression]\label{condregdef}
The {\em conditional (realizable) linear regression} task is the following. We 
are given access to examples from an arbitrary distribution $D$ over $\{0,1\}^n
\times\bbR^d\times\bbR$ for which there exists a $k$-DNF $c^*$ and $\vec{a}^*\in
\bbR^d$ 
such that
\begin{compactenum}
\item $\Pr_{(x,y,z)\in D}\left[|\langle \vec{a}^*,\vec{y}\rangle -z|\leq\epsilon
\mathlarger{|}c^*(\vec{x})=1\right]=1$ 
and 
\item $\Pr_{(x,y,z)\in D}[c^*(\vec{x})=1]\geq\mu$, 
\end{compactenum}
for some $\epsilon,\mu\in (0,1]$. Then with probability $1-\delta$, given
$\epsilon,\mu,\delta,$ and $\gamma$ as input, we are to find some $\vec{a}'\in
\bbR^d$ and $k$-DNF $c'$ such that
\begin{compactenum}
\item $\Pr_{(x,y,z)\in D}\left[|\langle\vec{a}',\vec{y}\rangle -z|\leq\epsilon
\mathlarger{|}c'(\vec{x})=1\right]\geq 1-\gamma$ and
\item $\Pr_{(x,y,z)\in D}[c'(\vec{x})=1]\geq \Omega\left(\left((1-\gamma)\frac
{\mu}{nd}\right)^k\right)$ for 
some $k$
\end{compactenum}
in time polynomial in $n,d,1/\mu,1/\epsilon,1/\gamma,$ and $1/\delta$.

If $\vec{a}^*$ is assumed to have at most $s$ nonzero entries and $\vec{a}'$ is
likewise required to have at most $s$ nonzero entries, then this is the
{\em conditional sparse linear regression} task with {\em sparsity} $s$.
\end{definition}
\noindent
We will also briefly consider the following variant that in some contexts may
be more natural.

\begin{definition}[Conditional $\ell_2$-linear regression]\label{l2regdef}
The {\em conditional $\ell_2$-linear regression} task is the following. We 
are given access to examples from an arbitrary distribution $D$ over $\{0,1\}^n
\times\{\vec{y}\in\bbR^d:\|\vec{y}\|_2\leq B\}\times [-B,B]$ for which there 
exists a $k$-DNF $c^*$ and $\vec{a}^*\in\bbR^d$ with $\|\vec{a}^*\|_2\leq B$ 
such that
\begin{compactenum}
\item $\bbE_{(x,y,z)\in D}\left[(\langle\vec{a}^*,\vec{y}\rangle -z)^2
\mathlarger{|}c^*(\vec{x})=1\right]\leq\epsilon$ 
and 
\item $\Pr_{(x,y,z)\in D}[c^*(\vec{x})=1]\geq\mu$, 
\end{compactenum}
for some $B\in\bbR^+$, $\epsilon,\mu\in (0,1]$. Then with probability 
$1-\delta$, given $B,\epsilon,\mu,\delta,$ and $\gamma$ as input, we are to find
some $\vec{a}'\in\bbR^d$ and $k$-DNF $c'$ such that
\begin{compactenum}
\item $\bbE_{(x,y,z)\in D}\left[(\langle\vec{a}',\vec{y}\rangle -z)^2
\mathlarger{|}c'(\vec{x})=1\right]\leq \poly(B,d,n)\epsilon$ and
\item $\Pr_{(x,y,z)\in D}[c'(\vec{x})=1]\geq \Omega\left( \left((1-\gamma)\frac
{\mu}{Bdn}\right)^k \right)$ for 
some $k$
\end{compactenum}
in time polynomial in $n,d,B,1/\mu,1/\epsilon,1/\gamma,$ and $1/\delta$.
\end{definition}

The restriction of $c$ to be a $k$-DNF is not arbitrary. Although we could
consider other classes of representations for $c$, it seems that essentially
any of the other standard hypothesis classes that we might naturally consider
here will lead to an intractable problem. This will follow since we can reduce 
the simpler problem of finding such conditions to our problem:

\begin{definition}[Conditional distribution search]
For a {\em representation class} $\calC$ of $c:\{0,1\}^n\to\{0,1\}$, the
{\em conditional distribution search problem} is as follows. Given access to
i.i.d. examples $(\vec{x}^{(1)},b^{(1)}),\ldots,(\vec{x}^{(m)},b^{(m)})$ from an
arbitrary distribution $D$ over $\{0,1\}^n\times\{0,1\}$ for which there exists
$c^*\in\calC$ such that $\Pr_{(x,b)\in D}[b=1|c^*(\vec{x})=1]=1$ and $\Pr_{(x,b)
\in D}[c^*(\vec{x})=1]\geq\mu$, with probability $1-\delta$, find some circuit 
$c'$ such that
\begin{compactenum}
\item $\Pr_{(x,b)\in D}[b=1|c'(\vec{x})=1]\geq 1-\gamma$ and
\item $\Pr_{(x,b)\in D}[c'(\vec{x})=1]\geq \Omega( ((1-\gamma)\mu/n)^k )$ for 
some $k$
\end{compactenum}
in time polynomial in $n,1/\mu,1/\gamma,$ and $1/\delta$.
\end{definition}

\begin{theorem}[Conditional distribution search reduces to conditional linear
regression]\label{realizable-reduction}
Suppose there is an algorithm that given access to examples from an arbitrary 
distribution $D'$ over $\{0,1\}^n\times\{0,1\}\times\{0,1\}$ for which there 
exists $c^*\in\calC$ and $a^*\in\bbR$ such that 
\[
\Pr_{(x,y,z)\in D'}\left[|a^*y-z|\leq\epsilon\mathlarger{|}c^*(\vec{x})=1\right]
=1\text{ and }\Pr_{(x,y,z)\in D'}[c^*(\vec{x})=1]\geq\mu,
\]
with probability $1-\delta$, finds some $a'\in\bbR$ and 
circuit $c'$ such that
\[
\Pr_{(x,y,z)\in D'}\left[|a'y-z|\leq\epsilon\mathlarger{|}c'(\vec{x})=1
\right]\geq 1-\gamma\text{ and }
\Pr_{(x,y,z)\in D'}[c'(\vec{x})=1]\geq \Omega( ((1-\gamma)\mu/n)^k )\text{ for
some }k
\]
in time polynomial in $n,1/\mu,1/\gamma,1/\epsilon$ and $1/\delta$.
Then there is a randomized polynomial-time algorithm for conditional 
distribution search for $\calC$.
\end{theorem}
\begin{proof}
Let $D$ be a distribution satisfying the hypotheses of the conditional
distribution search task for $\calC$, that is, for some $c^*\in\calC$,
\begin{compactenum}
\item $\Pr_{(x,b)\in D}[b=1|c^*(\vec{x})=1]=1$ and
\item $\Pr_{(x,b)\in D}[c^*(\vec{x})=1]\geq\mu$.
\end{compactenum}
Let $D'$ be the distribution over $\{0,1\}^n\times\{0,1\}\times\{0,1\}$ sampled 
as follows: given an example $(\vec{x},b)$ from $D$, if $b=1$ we produce $(\vec
{x},1,0)$ and otherwise we produce $(\vec{x},1,b')$ for $b'$ uniformly 
distributed over $\{0,1\}$. Notice that for $c^*$ and $a^*=0$, then whenever 
$c^*(\vec{x})=1$, $|a^*y-z|=0\leq 1/3$ over the entire support of the 
distribution; and, by assumption, $\Pr_{(x,y,z)\in D'}[c^*(\vec{x})=1]=
\Pr_{(x,b)\in D}[c^*(\vec{x})=1]\geq\mu$. So, the pair $a^*=0$ and $c^*$ 
certainly satisfy the conditions for our task for $\epsilon=1/3$. Therefore, by 
hypothesis, an algorithm for our task given access to $D'$ with $\epsilon=1/3$ 
and $\gamma'=\gamma/2$ must return $a'$ and a circuit $c'$ such that
\begin{compactenum}
\item $\Pr_{(x,y,z)\in D'}\left[|a'y-z|\leq 1/3 \mathlarger{|}c'(\vec{x})=1
\right]\geq 1-\gamma'$ and
\item $\Pr_{(x,y,z)\in D'}[c'(\vec{x})=1]\geq \Omega( ((1-\gamma')\mu/n)^k )$ 
for some $k$.
\end{compactenum}
But now, since the distribution we used is uniform over examples with $z=0$
and $z=1$ whenever $b=0$ (and $y\equiv 1$), it must be that whatever $a'$ is
returned, $|a'-z|>1/3$ with probability $1/2$ conditioned on $b=0$ in the
underlying draw from $D$. We must therefore actually have that 
\[
\frac{1}{2}\Pr_{(x,b)\in D}[b=0|c'(\vec{x})=1]\leq 
\Pr_{(x,y,z)\in D'}\left[|a'y-z|>1/3\mathlarger{|}c'(\vec{x})=1\right]\leq 
\frac{\gamma}{2}
\]
so indeed, also $\Pr_{(x,b)\in D}[b=1|c'(\vec{x})=1]\geq 1-\gamma$. Thus $c'$ is
as needed for a solution to the conditional distribution search problem. Since 
it is trivial to implement the sampling oracle for $D'$ given a sampling oracle 
for $D$, we obtain the desired algorithm.
\end{proof}

In turn now, algorithms for finding such conditions would yield algorithms 
for PAC-learning DNF~\cite{juba16}, which is currently suspected to be 
intractable (c.f. in particular work by Daniely and Shalev-Shwartz~\cite{dss16} 
for some strong consequences of learning DNF).

\begin{theorem}[Theorem 5 of \cite{juba16}]
If there exists an algorithm for the conditional distribution search problem
for conjunctions, then DNF is PAC-learnable in polynomial time.
\end{theorem}

Informally, therefore, an algorithm for conditional realizable linear regression
for conjunctions, or any class that can {\em express} conjunctions (instead of 
$k$-DNF) would yield a randomized polynomial time algorithm for PAC-learning 
DNF. This seems to rule out, in particular, the possibility of developing 
algorithms to perform regression under conditions described by halfspaces, 
decision trees, and so on. 

For conditional $\ell_2$-linear regression, a stronger conclusion holds:
such algorithms would solve the {\em agnostic} variant of the conditional
distribution search task, with a similar error bound: 

\begin{theorem}[Agnostic condition search reduces to conditional 
$\ell_2$-linear regression]\label{agnostic-reduction}
Suppose there is an algorithm that given access to examples from an arbitrary 
distribution $D'$ over $\{0,1\}^n\times\{0,1\}\times\{0,1\}$ for which there 
exists $c^*\in\calC$ and $a^*\in [0,1]$ such that $\bbE_{(x,y,z)\in D'}\left[
(a^*y-z)^2\mathlarger{|}c^*(\vec{x})=1\right]\leq\epsilon$ and $\Pr_{(x,y,z)\in 
D'}[c^*(\vec{x})=1]\geq\mu$, with probability $1-\delta$, finds some $a'$ and 
circuit $c'$ such that
\begin{compactenum}
\item $\bbE_{(x,y,z)\in D'}\left[(a'y-z)^2\mathlarger{|}c'(\vec{x})=1\right]
\leq p(n)\epsilon$ for some polynomial $p$ and
\item $\Pr_{(x,y,z)\in D'}[c'(\vec{x})=1]\geq \Omega( ((1-\gamma)\mu/n)^k )$ for
some $k$
\end{compactenum}
in time polynomial in $n,1/\mu,1/\gamma,1/\epsilon$ and $1/\delta$.
Then there is a randomized  polynomial-time algorithm for agnostic conditional 
distribution search for $\calC$: that is, if there exists $c\in\calC$ achieving
\begin{compactenum}
\item $\Pr_{(x,b)\in D}[b=1|c(\vec{x})=1]\geq 1-\epsilon$ and
\item $\Pr_{(x,b)\in D}[c(\vec{x})=1]\geq \mu$
\end{compactenum}
then the algorithm finds a circuit $c''$ achieving
\begin{compactenum}
\item $\Pr_{(x,b)\in D}[b=1|c''(\vec{x})=1]\geq 1-2p(n)\epsilon$ and
\item $\Pr_{(x,b)\in D}[c''(\vec{x})=1]\geq\Omega( ((1-\gamma)\mu/n)^k )$ for 
some $k$
\end{compactenum}
in time polynomial in $n,1/\mu,1/\gamma,1/\epsilon$ and $1/\delta$.
\end{theorem}
\begin{proof}
For a given distribution $D$ over $(x,b)$ satisfying the promise for conditional
distribution search, we use the same construction of $D'$ and reduction as in 
the proof of Theorem~\ref{realizable-reduction}. Here, we note that for $a^*=0$,
given that $\Pr_{(x,b)\in D}[b=1|c(\vec{x})=1]\geq 1-\epsilon$ for the $c$ 
assumed to exist for conditional distribution search
\[
\bbE_{(x,y,z)\in D'}\left[(0\cdot 1-z)^2\mathlarger{|}c(\vec{x})=1\right]\leq 
\frac{1}{2}\epsilon.
\]
Therefore, an algorithm for conditional $\ell_2$-linear regression must find
some $a'$ and circuit $c'$ such that $\Pr_{(x,y,z)\in D'}[c'(\vec{x})=1]\geq
\Omega( ((1-\gamma)\mu/n)^k )$ for some $k$ and
\[
\bbE_{(x,y,z)\in D'}\left[((a'-z)^2\mathlarger{|}c'(\vec{x})=1\right]\leq 
\frac{1}{2}p(n)\epsilon.
\]
Now, again, since $D'$ gives $z=0$ and $z=1$ equal probability whenever $b=0$,
we note that for such examples the expected value of $(a'-z)^2$ is minimized
by $a'=1/2$, where it achieves expected value $1/4$. Thus as $(a'-z)^2$ is 
surely nonnegative,
\[
\frac{1}{4}\Pr_{(x,b)\in D}[b=0|c'(\vec{x})=1]\leq\bbE_{(x,y,z)\in D'}\left[
(a'-z)^2\mathlarger{|}c'(\vec{x})=1\right]\leq 
\frac{1}{2}p(n)\epsilon
\]
so $c'$ indeed also achieves $\Pr_{(x,b)\in D}[b=1|c'(\vec{x})=1]\geq 1-2p(n)
\epsilon$.
\end{proof}

The restriction to constant sparsity is also key, as our problem contains as
a special case (when $\mu=1$, that is, when the trivial condition that takes
the entire population can be used) the standard sparse linear regression 
problem. Sparse linear regression for {\em constant} sparsity is easy, but when 
the sparsity is allowed to be large, the problem quickly becomes intractable:
In general, finding sparse solutions to linear equations is known to be 
$\NP$-hard~\cite{natarajan95}, and Zhang, Wainwright, and Jordan~\cite{zwj14}
extend this to bounds on the quality of sparse linear regression that is
achievable by polynomial-time algorithms, given that $\NP$ does not have
polynomial-size circuits.

\section{Algorithms for conditional sparse linear regression}

We now turn to stating and proving our main theorem. In what follows, we use
the following (standard) notation: $\Pi_{d_1,\ldots,d_s}$ denotes the projection
(of $\bbR^d$) to the $s$ coordinates $d_1,d_2,\ldots,d_s$ from $[d]$ (which 
denotes the integers $1,\ldots d$). For a set $S$, we let ${S\choose k}$
denote the subsets of $S$ of size exactly $k$.

\begin{algorithm}
\DontPrintSemicolon
\SetKwInOut{Input}{input}\SetKwInOut{Output}{output}
\SetKwInOut{Subroutines}{subroutines}
\Input{Examples $(\vec{x}^{(1)},\vec{y}^{(1)},z^{(1)}),\ldots,(\vec{x}^{(m)},
\vec{y}^{(m)},z^{(m)})$, target fit $\epsilon$ and fraction $(1-\gamma/2)\mu$.}
\Output{A $k$-DNF over $x_1,\ldots,x_n$ and linear predictor over $y_1,\ldots,
y_d$, or INFEASIBLE if none exist.}

\Begin{
\ForAll{$(d_1,\ldots,d_s)\in {[d]\choose s}$, $(\sigma_1,\ldots,\sigma_{s+1})\in
\{\pm 1\}^{s+1}$ and $(j_1,\ldots,j_{s+1})\in {[m]\choose s+1}$}
{
  Initialize $c$ to be the (trivial) $k$-DNF over all ${2n\choose k}$ terms of size $k$.\\
  Let $(\vec{a},\epsilon')$ be a solution to the following linear system: 
  $$\langle \vec{a},\Pi_{d_1,\ldots,d_s}\vec{y}^{(j_\ell)}\rangle
  -z^{(j_\ell)}=\sigma_\ell\epsilon'\text{ for }\ell =1,\ldots,{s+1}$$
  \lIf{$\epsilon'>\epsilon$}{continue to the next iteration.}
  \lFor{$j=1,\ldots,m$}
  {
    \If{$|\langle\vec{a},\Pi_{d_1,\ldots,d_s}\vec{y}^{(j)}\rangle-z^{(j)}|>
       \epsilon$} 
    {
      \lForAll{$T\in c$}
      {
        \lIf{$T(\vec{x}^{(j)})=1$}
        {
          Remove $T$ from $c$.
        }
      }
    }
  }
  \lIf{$\#\{j:c(\vec{x}^{(j)})=1\}>(1-\gamma/2)\mu m$}
  {
    \Return{$\vec{a}$ and $c$.}
  }
}
\Return{INFEASIBLE.}
}
\caption{Find-and-eliminate}\label{realizable-alg}
\end{algorithm}



\begin{theorem}[Realizable sparse regression -- full statement of Theorem~\ref
{mainthm}]
Suppose that $D$ is a joint probability distribution over $\vec{x}\in\{0,1\}^n$,
$\vec{y}\in\bbR^d$, and $z\in\bbR$ such that there is a $k$-DNF $c$ for
which for some $s$-sparse $\vec{a}\in\bbR^d$ 
\[
\Pr_{(x,y,z)\in D}\left[|\langle\vec{a},\vec{y}\rangle-z|\leq\epsilon
\mathlarger{|}c(\vec{x})=1\right]=1
\qquad\text{and}\qquad
\Pr_{(x,y,z)\in D}[c(\vec{x})=1]\geq\mu.
\]
Then given $\epsilon$, $\mu$, and $\delta$ in $(0,1)$ and $\gamma\in (0,1/2]$ 
and 
\[
m=O\left(\frac{1}{\mu\gamma}\left(s\log s+s\log d+n^k+
\log\frac{1}{\delta}\right)\right)
\]
examples from $D$, for any constants $s$ and $k$, Algorithm~\ref{realizable-alg}
runs in polynomial time in $n$, $d$, and $m$ ($=\poly(n,d,1/\mu,1/\gamma,\log 1/\delta)$) and 
finds an $s$-sparse $\vec{a}'$ and $k$-DNF $c'$ such that with probability 
$1-\delta$,
\[
\Pr_{(x,y,z)\in D}\left[|\langle\vec{a}',\vec{y}\rangle-z|\leq\epsilon
\mathlarger{|}c'(\vec{x})=1\right]\geq 1-\gamma
\qquad\text{and}\qquad
\Pr_{(x,y,z)\in D}[c'(\vec{x})=1]\geq(1-\gamma)\mu.
\]
\end{theorem}
\begin{proof}
It is clear that the algorithm runs for $O(d^sm^{s+1})$ iterations, where
each iteration (for constant $s$) runs in time polynomial in the bit length
of our examples and $O(mn^k)$. Thus, for constant $s$ and $k$, the algorithm 
runs in polynomial time overall, and it only remains to argue correctness.

We will first argue that the algorithm succeeds at returning some solution
with probability $1-\delta/3$ over the examples. We will then argue that any
solution returned by the algorithm is satisfactory with probability $1-2\delta
/3$ over the examples, thus leading to a correct solution with probability $1-
\delta$ overall.

\paragraph{Completeness part 1: Generating the linear rule.}
We first note that for $m\geq\frac{6}{\mu\gamma}\ln\frac{3}{\delta}$ examples, a
Chernoff bound guarantees that with probability $1-\delta/3$, there are at least
$(1-\gamma/2)\mu m$ examples satisfying the unknown condition $c$ in the sample.
Let $S$ be the set of examples satisfying $c$. 
Given the set of $s$ dimensions that are used in the sparse linear rule, we
set up a linear program in $s+1$ dimensions to minimize $\epsilon'$ subject to 
the constraints
\[
-\epsilon'\leq \langle\vec{a},\vec{y}^{(j)}\rangle -z^{(j)}\leq\epsilon' 
\text{ for }j\in S.
\]
It is well known (see, for example, Schrijver~\cite[Chapter 8]{schrijver86})
that the optimum value for any feasible linear program may be obtained at a 
{\em basic feasible solution,} i.e., a vertex of the polytope, given by 
satisfying $s+1$ of the constraints with equality. Since each constraint 
corresponds to an example and sign (for the lower or upper inequality), this 
means that we can obtain $\vec{a}$ by solving for $\vec{a}$ and $\epsilon'$ in 
the following linear system
\[
\langle\vec{a},\vec{y}^{(j_\ell)}\rangle -z^{(j_{\ell})}=\sigma_\ell\epsilon'
\text{ for }\ell=1,\ldots,s+1
\]
for some set of $s+1$ examples, $j_1,\ldots,j_{s+1}$ and $s+1$ signs
$\sigma_1,\ldots,\sigma_{s+1}$ corresponding to the tight constraints. Thus, 
when the algorithm uses the appropriate set of $s$ dimensions, the appropriate 
$s+1$ examples, and the appropriate $s+1$ signs, we will recover an $\vec{a}^*$ 
and $\epsilon^*$ such that for all $j\in S$, $|\langle\vec{a}^*,\vec{y}^{(j)}
\rangle -z^{(j)}|\leq\epsilon^*\leq \epsilon$.

\paragraph{Completeness part 2: Recovering a suitable condition given a rule.}
Now, given $\vec{a}^*$ such that for all $j\in S$, $|\langle\vec{a}^*,
\vec{y}^{(j)}\rangle -z^{(j)}|\leq\epsilon$, we observe that the algorithm 
identifies a $k$-DNF $h^*$ such that $h^*(\vec{x}^{(j)})=1$ for all $j\in S$. 
Indeed, the algorithm only eliminates a $k$-term $T$ for examples $j$ such that 
$|\langle\vec{a}^*,\vec{y}^{(j)}\rangle -z^{(j)}|>\epsilon$. Thus, it never 
eliminates any term appearing in $c$, and so in particular, $\Pr_{(x,y,z)\in D}
[h^*(\vec{x})=1]\geq\Pr_{(x,y,z)\in D}[c(\vec{x})=1]\geq\mu$. Moreover, since 
(as noted above, with probability $1-\delta/3$) there are at least $(1-\gamma/2)
\mu m$ examples satisfying $c$ in the sample, there are at least $(1-\gamma/2)
\mu m$ examples satisfying $h^*$. Thus, with probability $1-\delta/2$, when the 
algorithm considers the relevant $s$ dimensions in the support of $\vec{a}$ and 
considers an appropriate choice of $s+1$ examples to obtain a suitable $\vec
{a}^*$, it will furthermore obtain an $h^*$ that will lead the algorithm to 
terminate and return $\vec{a}^*$ and $h^*$.

\paragraph{Soundness: Generalization bounds.}
Next, we argue that any $\vec{a}'$ and $h'$ returned by the algorithm will 
suffice with probability $1-2\delta/3$ over the examples.

We will use the facts that
\begin{enumerate}
\item a union of $k$ hypothesis classes of VC-dimension $d$ has VC-dimension at 
most $O(d\log d+\log k)$ (for example, see \cite[Exercise 
6.11]{ssbd14}), 
\item linear threshold functions in $\bbR^s$ have VC-dimension $s+1$ (e.g., 
\cite[Section 9.1.3]{ssbd14}), and 
\item the composition of classes of VC-dimension $d_1$ and $d_2$ has 
VC-dimension at most $d_1+d_2$ (follows from \cite[Exercise 20.4]{ssbd14}).
\end{enumerate}
We now consider the class of disjunctions of a $k$-CNF over $\{0,1\}^n$ 
and (intersections of) linear threshold functions $\left[|\langle (\vec{a},-1),
(\vec{y},z)\rangle|\leq\varepsilon\right]$ for an $s$-sparse $\vec{a}$ over 
$\bbR^d$. By writing this class as a union over the $2^{n\choose k}$ $k$-CNFs 
and ${d\choose s}$ coordinate subsets of size $s$, we find that it has 
VC-dimension $O(s\log s+s\log(d/s)+n^k)=O(\log d+n^k)$ for constant $s$.%
\footnote{An exercise in Anthony and Biggs on the growth function for 
disjunctions of concepts~\cite[Chapter 8, Exercise 6]{ab92} also yields this 
easily.}

An optimal bound for sample complexity in terms of VC-dimension was recently
obtained by Hanneke~\cite{hanneke16} (superseding the earlier bounds, e.g.,
by Vapnik~\cite{vapnik82} and Blumer et al.~\cite{behw89}, although
these would suffice for us, too): in this case, given
\[
m=O\left(\frac{1}{\mu\gamma}\left(s\log s+s\log d+n^k+
\log\frac{1}{\delta}\right)\right)
\]
examples, if $[|\langle (\vec{a}',-1),(\vec{y},z)\rangle|\leq\varepsilon]\vee
\neg h'(\vec{x})$ is consistent with all of the examples, then with probability 
$1-\delta/3$ over the examples,
\[
\Pr_{(x,y,z)\in D}\left[\left(|\langle (\vec{a}',-1),(\vec{y},z)\rangle|\leq
\varepsilon\right)\vee\neg h'(\vec{x})\right]\geq 1-\mu\gamma/2
\]
or, equivalently,
\[
\Pr_{(x,y,z)\in D}\left[\left(|\langle (\vec{a}',-1),(\vec{y},z)\rangle| >
\varepsilon\right)\wedge h'(\vec{x})\right]\leq\mu\gamma/2.
\]
Now, since for $m\geq\frac{4}{\mu\gamma}\ln\frac{3}{\delta}$, with probability 
$1-\delta/3$, 
\[
\Pr_{(x,y,z)\in D}[h'(\vec{x})]\geq \frac{1-\gamma/2}{1+\gamma/2}\mu\geq 
(1-\gamma)\mu,
\]
we find that for our choice of $\vec{a}'$ and $h'$,
\begin{align*}
\Pr_{(x,y,z)\in D}\left[|\langle\vec{a}',\vec{y}\rangle -z|>\varepsilon
\mathlarger{|}h'(\vec{x})\right] &\leq 
\frac{\gamma}{2}\frac{1+\gamma/2}{1-\gamma/2}\\
\text{and so, } \Pr_{(x,y,z)\in D}\left[|\langle\vec{a}',\vec{y}\rangle -z|\leq
\varepsilon\mathlarger{|}h'(\vec{x})\right] &\geq 1-\gamma
\text{ since }\gamma\leq 1/2
\end{align*}
as needed.
\end{proof}


\section{Challenge: conditional dense, expected-error linear regression}\label
{densereg}

While sparsity is a highly desirable feature to have of a linear regression fit,
it may be the case that solutions are often not so sparse that Algorithm~\ref
{realizable-alg} is truly efficient. Moreover, we may also wish for an algorithm
that handles an {\em expected error} variant of the regression task. Our 
technique certainly does not address either of these concerns. The following 
simple algorithm illustrates the best technique we currently have for either
dense regression or expected error regression.

\begin{algorithm}
\SetKwInOut{Input}{input}\SetKwInOut{Output}{output}
\SetKwInOut{Subroutines}{subroutines}
\Input{Examples $(\vec{x}^{(1)},\vec{y}^{(1)},z^{(1)}),\ldots,(\vec{x}^{(m)},
\vec{y}^{(m)},z^{(m)})$, target fit $\epsilon$.}
\Output{A $k$-DNF over $x_1,\ldots,x_n$ and linear predictor over $y_1,\ldots,
y_d$.}

\Begin{
Initialize $c=\bot$, $\mu^*=0$.\\
\ForAll{Terms $T$ of size $k$ over $x_1,\ldots,x_n$}
{
  Put $S(T)=\{j:T(\vec{x}^{(j)})=1\}$.\\
  Let $\vec{a}$ minimize the squared-error on $(\vec{y}^{(j)},z^{(j)})$ over 
  $j\in S(T)$ subject to $\|\vec{a}\|_2\leq B$.\\ 
  \If{ $\frac{1}{m}\sum_{j\in S(T)}(\langle\vec{a},\vec{y}^{(j)}\rangle-
       z^{(j)})^2\leq 4\mu\epsilon$  and $|S(T)|\geq \mu^*m$}
  {
    Put $c=T$ and $\mu^*=|S(T)|/m$.
  }
}
\Return{$c$ and $\vec{a}$}
}
\caption{Dense Expected-error Regression Pigeonhole (DERP)}\label{derp}
\end{algorithm}

\begin{theorem}
Algorithm~\ref{derp} solves the conditional $\ell_2$-linear regression task:
given access to a joint distribution $D$ over $\vec{x}\in\{0,1\}^n$,
$\vec{y}\in\bbR^d$ with $\|\vec{y}\|_2\leq B$, and $z\in[-B,B]$ 
such that there is a $k$-DNF $c$ and $\vec{a}\in\bbR^d$ with $\|\vec{a}\|_2\leq
B$ such that
\[
\bbE_{(x,y,z)\in D}\left[(\langle\vec{a},\vec{y}\rangle-z)^2\mathlarger{|}
c(\vec{x})=1\right]\leq\epsilon
\qquad\text{and}\qquad
2\mu\geq\Pr_{(x,y,z)\in D}[c(\vec{x})=1]\geq\mu
\]
and given $B$, $k$, $\epsilon$, $\mu$, and $\delta\in (0,1)$, using
\[
m=O\left(\frac{B^8n^k}{\mu\epsilon}\left(k\log n+\log\frac{1}{\delta}\right)\right)
\]
examples from $D$, for any constant $k$, Algorithm~\ref{derp} runs in polynomial
time and finds a $\vec{a}'$ and $k$-DNF $c'$ such
that with probability $1-\delta$,
\[
\bbE_{(x,y,z)\in D}\left[(\langle\vec{a}',\vec{y}\rangle-z)^2\mathlarger{|}
c'(\vec{x})=1\right]\leq O(n^k\epsilon)
\qquad\text{and}\qquad
\Pr_{(x,y,z)\in D}[c'(\vec{x})=1]\geq\Omega(\mu/n^k).
\]
\end{theorem}
Note that we can find such an estimate for $\mu$ by binary search.

\begin{proof}
We first observe that in particular, since for any $T$ the objective function
\[
\sum_{j\in S(T)}(\langle\vec{a},\vec{y}^{(j)}\rangle-z^{(j)})^2
\]
is convex, as is the set of $\vec{a}$ of $\ell_2$-norm at most $B$, the main 
step of the algorithm is a convex optimization problem that can be solved in 
polynomial time, for example by gradient descent (see, e.g., \cite[Chapter 14]
{ssbd14}). Thus, the algorithm can be implemented in polynomial time as claimed.

We next turn to correctness. Let $c^*$ be the $k$-DNF promised by the theorem 
statement. By the pigeonhole principle, there must be some term $T^*$ of $c^*$ 
such that $\Pr[T^*(\vec{x})=1]\geq\mu/{2n\choose k}$. Observe that for the rule 
$\vec{a}^*$ promised to exist,
\[
\bbE_{D}\left[(\langle\vec{a}^*,\vec{y}\rangle -z)^2\mathlarger{|}
T^*(\vec{x})=1\right]\Pr_{D}[T^*(\vec{x})=1]\leq
\bbE_{D}\left[(\langle\vec{a}^*,\vec{y}\rangle -z)^2\mathlarger{|}
c^*(\vec{x})=1\right]\Pr_{D}[c^*(\vec{x})=1]
\leq \epsilon\cdot 2\mu.
\]
For a suitable choice of leading constant in the number of examples,
a (multiplicative) Chernoff bound yields that with probability $1-\delta/4$, at 
least $m\Pr[T^*(\vec{x})=1]/2$ examples satisfy $T^*$ and noting that $(\langle 
\vec{a}^*,\vec{y}\rangle -z)^2\in [0,2B^4]$, with probability $1-\delta/4$,
\[
\frac{1}{m}\sum_{j=1}^m(\langle\vec{a}^*,\vec{y}\rangle -z)^2T^*(\vec{x})\leq 4
\mu\epsilon
\]
Thus, the $\vec{a}'$ minimizing the squared error on the set of examples also
achieves $\frac{1}{m}\sum_{j:T^*(\vec{x}^{(j)})=1}(\langle\vec{a}',\vec{y}^{(j)}
\rangle-z^{(j)})^2\leq 4\mu\epsilon$  as needed, so with probability $1-\delta/
2$, at least $T^*$ is considered for $c$ and the algorithm produces some $c$ and
$\vec{a}$ as output.

To see that any such $T$ and $\vec{a}$ is satisfactory, we first note that any 
$T$ we produce as output must satisfy at least as many examples as $T^*$ by 
construction, so $T$ must satisfy at least
\[
\Pr_D[T^*(\vec{x})=1]m/2\geq\Omega\left(\frac{B^8}{\epsilon}\left(k\log n+
\log\frac{1}{\delta}\right)\right)
\]
examples. In particular, this is at least $\mu m/2{2n\choose k}$ examples, and a
Chernoff bound guarantees that for suitable constants, with probability $1-
\delta/4{2n\choose k}$, no $T$ with $\Pr_D[T(x)=1]<\mu/4{2n\choose k}$ satisfies
so many examples. Next, simply note that if for the best $a$ for $T$ with 
$\|\vec{a}\|_2\leq B$,  $\bbE_D\left[(\langle\vec{a},\vec{y}\rangle-z)^2
\mathlarger{|}T(\vec{x})=1\right]\Pr[T(x)=1]>8\mu\epsilon$, then since 
$\|\vec{y}\|_2\leq B$, $z^2\leq B^2$, and the loss function is $B$-Lipschitz on 
this domain, a Rademacher bound (see, for example, \cite[Theorem 26.12]{ssbd14})
guarantees that with probability $1-\delta/4{2n\choose k}$, for any such $\vec
{a}$,
\[
\frac{1}{m}\sum_{j:T(\vec{x}^{(j)})=1}(\langle\vec{a},\vec{y}^{(j)}\rangle-
z^{(j)})^2>4\mu\epsilon 
\]
and $T$ will not be considered. A union bound over both events for all such $T$ 
establishes that any $T$ that is returned has, with probability $1-\delta/2$, 
both 
\[
\Pr_D[T(\vec{x})=1]\geq\frac{\mu}{4{2n\choose k}}\text{ and }\bbE_D\left[
(\langle\vec{a},\vec{y}\rangle-z)^2\mathlarger{|}T(\vec{x})=1\right]
\Pr_D[T(\vec{x})=1]\leq 8\mu\epsilon 
\]
and thus is as needed. Therefore, overall, with probability $1-\delta$, the 
algorithm considers at least $T^*$ as a candidate to output, and outputs a 
suitable term $T$ and vector $\vec{a}$.
\end{proof}

The main defect of Algorithm~\ref{derp} is that in general it only recovers a
condition with a $\Omega(1/n^k)$-fraction of the possible probability mass of 
the best $k$-DNF condition. This is in stark contrast to both Algorithm~\ref
{realizable-alg} and all of the earlier positive results for condition 
identification~\cite{juba16}, in which we find a condition with probability at 
least a $(1-\gamma)$-fraction of that of the best condition, for any $\gamma$ we
choose. Indeed, we are most interested in the case where the probability of this
event is relatively small and thus a $1/n^k$-fraction is extremely 
small. The main challenge here is to develop an algorithm for the dense and/or 
expected-error regression problem that similarly identifies a condition with 
probability that is a $(1-\gamma)$-fraction of that of the best condition.

Of course, the $O(n^k)$ blow-up in the expected error is also undesirable, but
as indicated by Theorem~\ref{agnostic-reduction}, this is the same difficulty 
encountered in {\em agnostic learning}. Naturally, minimizing the amount by
which constraints are violated is generally a harder problem than finding a 
solution to a system of constraints, and this is reflected in the quality of 
results that have been obtained. The results for such agnostic condition 
identification of $k$-DNFs in the previous work by Juba~\cite{juba16} suffered a
similar blow-up in the error, as indeed do the state-of-the-art algorithms for 
agnostic supervised learning for disjunctive classifiers by Awasthi et al.~\cite
{abs10}. Although Awasthi et al. managed to reduce the increase in error for 
this somewhat different problem to a $\sim n^{k/3}$-factor, it still increases 
polynomially with the number of attributes. We note briefly that a variant of 
Algorithm~\ref{derp} in which we seek $\vec{a}$ satisfying $|\langle\vec{a},\vec
{y}^{(j)}\rangle -z^{(j)}|\leq\epsilon$ for all $j$ satisfying a candidate term 
$T$ solves the ``realizable regression'' (sup norm) variant of Definition~\ref
{condregdef} for dense regression, and {\em does not} suffer this increase of 
the error. 

\section*{Acknowledgements}
I thank Madhu Sudan for originally suggesting the joint problem of learning 
under conditional distributions. I also thank Ben Moseley for many helpful
discussions about these problems.

\bibliographystyle{plain}
\bibliography{robust}

\begin{thebibliography}{10}

\bibitem{amstv96}
Rakesh Agrawal, Heikki Mannila, Ramakrishnan Srikant, Hannu Toivonen, and
  A.~Inkeri Verkamo.
\newblock Fast discovery of association rules.
\newblock In {\em Advances in Knowledge Discovery and Data Mining}, chapter~12,
  pages 307--328. MIT Press, Cambridge, MA, 1996.

\bibitem{ab92}
Martin Anthony and Norman Biggs.
\newblock {\em Computational Learning Theory}.
\newblock Number~30 in Cambridge Tracts in Theoretical Computer Science.
  Cambridge University Press, New York, NY, 1992.

\bibitem{abs10}
Pranjal Awasthi, Avrim Blum, and Or~Sheffet.
\newblock Improved guarantees for agnostic learning of disjunctions.
\newblock In {\em Proc. 23rd COLT}, pages 359--367, 2010.

\bibitem{behw89}
Anselm Blumer, Andrzej Ehrenfeucht, David Haussler, and Manfred~K. Warmuth.
\newblock Learnability and the {Vapnik}-{Chervonenkis} dimension.
\newblock {\em J. ACM}, 36(4):929--965, 1989.

\bibitem{bb05}
Nader~H. Bshouty and Lynn Burroughs.
\newblock Maximizing agreements with one-sided error with applications to
  heuristic learning.
\newblock {\em Machine Learning}, 59(1--2):99--123, 2005.

\bibitem{dss16}
Amit Daniely and Shai Shalev-Shwartz.
\newblock Complexity theoretic limtations on learning {DNF's}.
\newblock In {\em Proc. 29th COLT}, volume~49 of {\em JMLR Workshops and
  Conference Proceedings}, pages 1--16. 2016.

\bibitem{fb81}
Martin~A. Fischler and Robert~C. Bolles.
\newblock Random sample consensus: A paradigm for model fitting with
  applications to image analysis and automated cartography.
\newblock {\em Communications of the ACM}, 24(6):381--395, 1981.

\bibitem{ff99}
Jerome~H. Friedman and Nicholas~I. Fisher.
\newblock Bump hunting in high-dimensional data.
\newblock {\em Statistics and Computing}, 9(2):123--143, 1999.

\bibitem{hanneke16}
Steve Hanneke.
\newblock The optimal sample complexity of {PAC} learning.
\newblock {\em JMLR}, 17(38):1--15, 2016.

\bibitem{huber81}
Peter~J. Huber.
\newblock {\em Robust Statistics}.
\newblock John Wiley \& Sons, New York, NY, 1981.

\bibitem{juba16}
Brendan Juba.
\newblock Learning abductive reasoning using random examples.
\newblock In {\em Proc. 30th AAAI}, pages 999--1007, 2016.

\bibitem{kkm12}
Adam~Tauman Kalai, Varun Kanade, and Yishay Mansour.
\newblock Reliable agnostic learning.
\newblock {\em JCSS}, 78:1481--1495, 2012.

\bibitem{kt14}
Varun Kanade and Justin Thaler.
\newblock Distribution-independent reliable learning.
\newblock In {\em Proc. 27th COLT}, volume~35 of {\em JMLR Workshops and
  Conference Proceedings}, pages 3--24, 2014.

\bibitem{natarajan95}
B.~K. Natarajan.
\newblock Sparse approximate solutions to linear systems.
\newblock {\em SIAM J. Comput.}, 24(2):227--234, 1995.

\bibitem{pv88}
Leonard Pitt and Leslie~G. Valiant.
\newblock Computational limitations on learning from examples.
\newblock {\em J. ACM}, 35(4):965--984, 1988.

\bibitem{rgh+15}
Avi Rosenfeld, David~G. Graham, Rifat Hamoudi, Rommell Butawan, Victor Eneh,
  Saif Kahn, Haroon Miah, Mahesan Niranjan, and Laurence~B. Lovat.
\newblock {MIAT}: A novel attribute selection approach to better predict upper
  gastrointestinal cancer.
\newblock In {\em Proc. IEEE International Conference on Data Science and
  Advanced Analytics (DSAA)}, pages 1--7, 2015.

\bibitem{rl87}
Peter~J. Rousseeuw and Annick~M. Leroy.
\newblock {\em Robust Regression and Outlier Detection}.
\newblock John Wiley \& Sons, New York, NY, 1987.

\bibitem{schrijver86}
Alexander Schrijver.
\newblock {\em Theory of Linear and Integer Programming}.
\newblock Wiley-Interscience Series in Discrete Mathematics and Optimization.
  John Wiley \& Sons, 1986.

\bibitem{ssbd14}
Shai Shalev-Shwartz and Shai Ben-David.
\newblock {\em Understanding Machine Learning: From Theory to Algorithms}.
\newblock Cambridge University Press, New York, NY, 2014.

\bibitem{vapnik82}
Vladimir Vapnik.
\newblock {\em Estimation of Dependencies Based on Empirical Data}.
\newblock Springer, New York, NY, 1982.

\bibitem{zwj14}
Yuchen Zhang, Martin~J. Wainwright, and Michael~I. Jordan.
\newblock Lower bounds on the performance of polynomial-time algorithms for
  sparse linear regression.
\newblock In {\em Proc. 27th COLT}, volume~35 of {\em JMLR Workshops and
  Conference Proceedings}, pages 921--948. 2014.

\end{thebibliography}

\end{document}